\documentclass[10pt,twocolumn,letterpaper]{article}

\usepackage{times}
\usepackage{epsfig}
\usepackage{graphicx}
\usepackage{amsmath}
\usepackage{amssymb}
\usepackage{amsmath}
\usepackage{amsthm}
\usepackage{amsfonts}
\usepackage{bbm}
\newtheorem{mythm}{Theorem}

\usepackage{subcaption}
\usepackage[linesnumbered,ruled,vlined]{algorithm2e}
\usepackage{mathtools}
\usepackage[shortlabels]{enumitem}
\newcommand{\etal}{\textit{et al.}}
\DeclareMathOperator*{\argmin}{arg\,min}

\def\wacvPaperID{1070} 

\begin{document}

\title{$\ell_0$ Regularized Structured Sparsity Convolutional Neural Networks}

\author{Kevin Bui \\
Department of Mathematics\\
University of California, Irvine\\
{\tt\small kevinb3@uci.edu}
\and
Fredrick Park \\
Department of Mathematics \& Computer Science\\
Whittier College\\
{\tt\small fpark@whittier.edu}
\and
Shuai Zhang \\
Qualcomm AI \\
{\tt\small shuazhan@qti.qualcomm.com}
\and
Yingyong Qi\\
Qualcomm AI \\
{\tt\small yqi@uci.edu}
\and
Jack Xin \\
Department of Mathematics\\
University of California, Irvine\\
{\tt\small jxin@math.uci.edu}
}

\maketitle

\begin{abstract}
   Deepening and widening convolutional neural networks (CNNs) significantly increases the number of trainable weight parameters by adding more convolutional layers and feature maps per layer, respectively. By imposing inter- and intra-group sparsity onto the weights of the layers during the training process, a compressed network can be obtained with accuracy comparable to a dense one. In this paper, we propose a new variant of sparse group lasso that blends the $\ell_0$ norm onto the individual weight parameters and the $\ell_{2,1}$ norm onto the output channels of a layer. To address the non-differentiability of the $\ell_0$ norm, we apply variable splitting resulting in an algorithm that consists of executing stochastic gradient descent followed by hard thresholding for each iteration. Numerical experiments are demonstrated on LeNet-5 and wide-residual-networks for MNIST and CIFAR 10/100, respectively. They showcase the effectiveness of our proposed method in attaining superior test accuracy with network sparsification on par with the current state of the art. 
\end{abstract}

\section{Introduction}
Deep neural networks (DNNs) have proven to be advantageous for numerous modern computer vision tasks involving image or video data. In particular, convolutional neural networks (CNNs) yield highly accurate models with applications in image classification \cite{krizhevsky2012imagenet,Simonyan2015VeryDC,he2016deep,zagoruyko2016wide}, semantic segmentation \cite{long2015fully,chen2017deeplab}, and object detection \cite{ren2015faster,huang2017speed,parkhi2015deep}. 
These large models often contain millions or even billions of weight parameters that often exceed the number of training data. This is a double-edged sword since on one hand, large models allow for high accuracy, while on the other, they contain many redundant parameters that lead to overparametrization. Overparametrization is a well-known phenomenon in DNN models \cite{denton2014exploiting,ba2014deep} that results in overfitting, learning useless random patterns in data \cite{zhang2016understanding}, and having inferior generalization. Additionally, these models also possess exorbitant computational and memory demands during both training and inference. As a result, they may not be applicable for devices with low computational power and memory.

Resolving these problems requires compressing the networks through sparsification and pruning. Although removing weights might affect the accuracy and generalization of the models, previous works \cite{Louizos2017LearningSN,han2015deep,ullrich2017soft,molchanov2017variational} demonstrated that many networks can be substantially pruned with negligible effect on accuracy. There are many systematic approaches to achieving sparsity in DNNs.

Han \etal \cite{han2015learning} proposed to first train a dense network, prune it afterward by setting the weights to zero if below a fixed threshold, and retrain the network with the remaining weights. Jin \etal \cite{jin2016training} extended this method by restoring the pruned weights, training the network again, and repeating the process. Rather than pruning by thresholding, Aghasi \etal \cite{aghasi2017net} proposed Net-Trim, which prunes an already trained network layer-by-layer using convex optimization in order to ensure that the layer inputs and outputs remain consistent with the original network.  For CNNs in particular, filter or channel pruning is preferred because it  significantly reduces the amount of weight parameters required compared to individual weight pruning. Le \etal \cite{li2016pruning} calculated the sums of absolute weights of the filters of each layer and pruned the ones with the smallest weights. Hu \etal \cite{hu2016network} proposed a metric called average percentage of zeroes for channels to measure their redundancies and pruned those with highest values for each layer. Zhuang \etal \cite{zhuang2018discrimination} developed discrimination-aware channel pruning that selects channels that contribute to the discriminative power of the network.

An alternative approach to pruning a dense network is learning a compressed structure from scratch. A conventional approach is to optimize the loss function equipped with either the $\ell_1$ or $\ell_2$ regularization, which drives the weights to zero or to very small values during training. To learn which groups of weights (e.g., neurons, filters, channels) are necessary, group regularization, such as group lasso \cite{yuan2006model} and sparse group lasso \cite{simon2013sparse}, are equipped to the loss function. Alvarez and Salzmann \cite{alvarez2016learning} and Scardapane \etal \cite{scardapane2017group} applied group lasso and sparse group lasso to various architectures and obtained compressed networks with comparable or even better accuracy. Instead of sharing features among the weights as suggested by group sparsity, exclusive sparsity \cite{zhou2010exclusive} promotes competition for features between different weights. This method was investigated by Yoon and Hwang \cite{yoon2017combined}. In addition, they combined it with group sparsity and demonstrated that this combination resulted in compressed networks with better performance than their original. Non-convex regularization has also been examined. Louizos \etal \cite{Louizos2017LearningSN} proposed a practical algorithm using probabilistic methods to perform $\ell_0$ regularization on neural networks. Ma \etal \cite{ma2019transformed} proposed integrated transformed $\ell_1$, a convex combination of transformed $\ell_1$ and group lasso, and compared its performance against the aforementioned group regularization methods. 

In this paper, we propose a group regularization method that balances both group lasso and $\ell_0$ regularization: it is a variant of sparse group lasso that replaces the $\ell_1$ penalty term with the $\ell_0$ penalty term. This proposed group regularization method is presumed to yield a better performing, compressed network than sparse group lasso since $\ell_1$ is a convex relaxation of $\ell_0$. We develop an algorithm to optimize loss functions equipped with the proposed regularization term for DNNs.   
\section{Model and Algorithm}
\subsection{Preliminaries}
Given a training dataset consisting of $N$ input-output pairs $\{(x_i, y_i)\}_{i=1}^{N}$, the weight parameters of a DNN are learned by optimizing the following objective function:
\begin{align}\label{eq:min_problem}
    \min_{W} \frac{1}{N} \sum_{i=1}^N \mathcal{L}(h(x_i, W), y_i) + \lambda \mathcal{R}(W), 
\end{align}
where 
\begin{itemize}
    \item $W$ is the set of weight parameters of the DNN. 
    \item $\mathcal{L}(\cdot, \cdot) \geq 0$ is the loss function that compares the prediction $h(x_i, W)$ with the ground-truth output $y_i$. Examples include cross-entropy loss function for classification and mean-squared error for regression.
    \item $h(\cdot, \cdot)$ is the output of the DNN used for prediction. 
    \item $\lambda>0$ is a regularization parameter for $\mathcal{R}(\cdot)$.
    \item $\mathcal{R}(\cdot)$ is the regularizer on the set of weight parameters $W$. 
\end{itemize}

The most common regularizer used for DNN is $\|\cdot\|_2^2$, also known as weight decay. It prevents overfitting and improves generalization because it enforces the weights to decrease proportional to their magnitudes \cite{krogh1992simple}. Sparsity can be imposed by pruning weights whose magnitudes are below a certain threshold at each iteration during training. However, an alternative regularizer is the $\ell_1$ norm $\|\cdot\|_1$, also known as lasso \cite{tibshirani1996regression}. $\ell_1$ norm is the tightest convex relaxation of the $\ell_0$ norm and it yields a sparse solution that is found on the corners of the 1-norm ball. Unfortunately, element-wise sparsity by $\ell_1$ or $\ell_2$ regularization in CNNs may not yield meaningful speedup as the number of filters and channels required for computation and inference may remain the same \cite{wen2016learning}. 

To determine which filters or channels are relevant in each layer, group sparsity using group lasso is considered. Suppose a DNN has $L$ layers, so the set of weight parameters $W$ is divided into $L$ sets of weights: $W = \{W_l\}_{l=1}^L$. The weight set of each layer $W_l$ is divided into $N_l$ groups (e.g., channels or filters): $W_{l} = \{w_{l,g}\}_{g=1}^{N_l}$. Group lasso applied to $W_l$ is formulated as
\begin{align}
    \mathcal{R}_{GL}(W_l)&= \sum_{g=1}^{N_l} \sqrt{|w_{l,g}|} \|w_{l,g}\|_2\\ \nonumber &= \sum_{g=1}^{N_l} \sqrt{|w_{l,g}|} \sqrt{\sum_{i=1}^{|w_{l,g}|} w_{l,g,i}^2},
\end{align}
where $w_{l,g,i}$ corresponds to the weight parameter with index $i$ in group $g$ in layer $l$, and the term $\sqrt{|w_{l,g}|}$ ensures that each group is weighed uniformly. This regularizer imposes the $\ell_2$ norm on each group, forcing weights of the same groups to decrease altogether at every iteration during training. As a result, groups of weights are pruned when their $\ell_2$ norms are negligible, resulting in a highly compact network compared to element-sparse networks. 

To obtain an even sparser network, element-wise sparsity and group sparsity can be combined and applied together to the training of DNNs. One regularizer that combines these two types of sparsity is sparse group lasso, which is formulated as
\begin{align}\label{eq:sgl}
    \mathcal{R}_{SGL}(W_l) = \mathcal{R}_{GL}(W_l) + \|W_l\|_1,
\end{align}
where
\begin{align*}
    \|W_l\|_1 = \sum_{g=1}^{|N_l|} \sum_{i=1}^{|w_{l,g}|} |w_{l,g,i}|. 
\end{align*}
Sparse group lasso simultaneously enforces group sparsity by having $\mathcal{R}_{GL}(\cdot)$ and element-wise sparsity by having $\|\cdot\|_1$. 
\subsection{Proposed Regularizer: Sparse Group L$_0$asso}
We recall that the $\ell_1$ norm is a convex relaxation of the $\ell_0$ norm, which is non-convex and discontinuous. In addition, any $\ell_0$-regularized problem is NP-hard. These properties make developing convergent and tractable algorithms for $\ell_0$-regularized problems difficult, thereby making $\ell_1$-regularized problems better alternatives to solve. However, the $\ell_0$-regularized problems have their advantages over their $\ell_1$ counterparts. For example, they are able to recover better sparse solutions than do $\ell_1$-regularized problems in various applications, such as compressed sensing \cite{lu2013sparse}, image restoration \cite{bao2016image,chan2003wavelet,dong2013efficient,zhang2013}, MRI reconstruction \cite{trzasko2007sparse}, and machine learning \cite{lu2013sparse,yuan2017gradient}. Used to solve $\ell_1$ minimization, the soft-thresholding operator $\mathcal{S}_{\lambda}(c) = \text{sign}(c) \max\{|c|- \lambda, 0\}$ yields a biased estimator \cite{fan2001variable}.

Due to the advantages and recent successes of $\ell_0$ minimization, we propose to replace the $\ell_1$ norm in \eqref{eq:sgl} with the $\ell_0$ norm
\begin{align}
    \|W_l\|_0 = \sum_{g=1}^{|N_l|} \sum_{i=1}^{|w_{l,g}|} |w_{l,g,i}|_0,
\end{align}
where
\begin{align*}
    |w|_0 = \begin{cases}
    1 \text{ if } w \neq 0 \\
    0 \text{ if } w = 0.
    \end{cases}
\end{align*}
Hence, we propose a new regularizer called sparse group $\ell_0$asso defined by
\begin{align}\label{eq:SGL0}
  \mathcal{R}_{SGL_0}(W_l) = \mathcal{R}_{GL}(W_l) + \|W_l\|_0.   
\end{align}
Using this regularizer, we expect to obtain a better sparse network than from using sparse group lasso. \subsection{Notation}
Before discussing the algorithm, we summarize notations that we will use to save space. They are the following:
\begin{itemize}
    \item If $V = \{V_l\}_{l=1}^L$ and $W =\{W_l\}_{l=1}^{L}$, then $(V, W) \coloneqq (\{V_l\}_{l=1}^{L}, \{W_l\}_{l=1}^L) = (V_1, \ldots, V_L, W_1, \ldots, W_L)$.
    \item For $V = \{V_l\}_{l=1}^L$, $V_{<l} = (V_1, \ldots, V_{l-1})$ and $V_{>l} = (V_{l+1}, \ldots, V_L)$. Both $V_{\leq l}$ and $V_{\geq l}$ are defined similarly.
    \item $V^+ \coloneqq V^{k+1}$. 
    \item $\tilde{\mathcal{L}}(W) \coloneqq \frac{1}{N} \sum_{i=1}^N \mathcal{L}(h(x_i, W), y_i)$.
\end{itemize}

\subsection{Numerical Optimization}
We develop an algorithm to solve \eqref{eq:min_problem} with the sparse group $\ell_0$asso regularizer \eqref{eq:SGL0}. So, with $W = \{W_l\}_{l=1}^L$, the minimization problem we solve is
\begin{align}\label{eq:SGL0_prob}
    &\min_{W} \tilde{\mathcal{L}}(W) + \lambda  \sum_{l=1}^L\mathcal{R}_{SGL_0}(W_l) \\
    &=\nonumber \tilde{\mathcal{L}}(W) + \lambda  \sum_{l=1}^L\left( \mathcal{R}_{GL}(W_l)+\|W_l\|_0 \right).
\end{align}
Throughout this paper, we assume that $\mathcal{L}$ is continuously differentiable with respect to $W_l$ for each $l = 1, \ldots, L$.
Because finding the subderivative of the objective problem is difficult due to the $\ell_0$ norm, we need to figure out a method to solve it. By introducing an auxiliary variable $V = \{V_l\}_{l=1}^L$, we have a constrained optimization problem
\begin{equation}
\begin{aligned}
    &\min_{V,W}  & &\tilde{\mathcal{L}}(W) + \lambda \sum_{l=1}^L  \left( \mathcal{R}_{GL}(W_l)+\|V_l\|_0 \right)\\
    & \text{s.t.}& &V=W.
\end{aligned}
\end{equation}
The constraint can be relaxed by adding a quadratic penalty term with $\beta > 0$ so that we have
\begin{equation}\label{eq:relaxed_min_prob}
\begin{aligned}
    &\min_{V,W} F_{\beta}(V,W) \\&\coloneqq \tilde{\mathcal{L}}(W) +   \sum_{l=1}^L \Big[\lambda\left( \mathcal{R}_{GL}(W_l)+\|V_l\|_0 \right)\\
    & + \frac{\beta}{2} \|V_l-W_l\|_2^2 \Big].
\end{aligned}
\end{equation}
With $\beta$ fixed, \eqref{eq:relaxed_min_prob} can be solved by alternating minimization:
\begin{subequations}
\begin{align}
\begin{split} \label{eq:W_update}
W_l^{k+1} = \argmin_{W_l} F_{\beta}(V^k,W_{<l}^+, W_l, W_{>l}^k) \\ \text{ for } l=1, \ldots, L
\end{split}\\
\begin{split}\label{eq:V_update}
V^{k+1} =  \argmin_{V} F_{\beta}(V, W^{k+1}).
\end{split}
\end{align}
\end{subequations}
We explicitly update $W_l$ by gradient descent and $V_l$ by hard-thresholding:
\begin{subequations}
\begin{align} \begin{split}\label{eq:gradient_descent}
    &W_l^{k+1} = W_l^k - \gamma \Big(  \nabla_{W_l} \tilde{\mathcal{L}}(W)\\  &+ \lambda \partial{\mathcal{R}_{GL}(W_l^k)} - \beta  (V_l^k- W_l^k)  \Big) 
    \text{ for } l = 1, \ldots, L \end{split} \\
    \begin{split}\label{eq:hard_threshold}
    &V^{k+1} = \mathcal{H}_{\sqrt{2 \lambda/\beta}}(W^{k+1}), 
    \end{split}
\end{align}
\end{subequations}
where $\gamma$ is the learning rate, $\partial{\mathcal{R}_{GL}}$ is the subdifferential of $\mathcal{R}_{GL}$, and $ \mathcal{H}_{\sqrt{2 \lambda/\beta}}(\cdot)$ is the element-wise hard-thresholding operator:
\begin{align}
        H_{\sqrt{2 \lambda/\beta}}(w_i) = \begin{cases}
    0 &\text{ if } |w_i| \leq \sqrt{2 \lambda/\beta} \\
    w_i &\text{ if } |w_i| > \sqrt{2 \lambda/ \beta}.
    \end{cases}
\end{align}
In practice, \eqref{eq:gradient_descent} is performed using stochastic gradient descent (or one of its variants) with mini-batches due to the large-size computation dealing with the amount of data and weight parameters that a typical DNN has. 

After presenting an algorithm that solves the quadratic penalty problem \eqref{eq:relaxed_min_prob}, we now present an algorithm to solve \eqref{eq:SGL0_prob}.  We solve a sequence of quadratic penalty problems \eqref{eq:relaxed_min_prob} with $\beta \in \{\beta_j\}_{j=1}^{\infty}$ such that $\beta_j \uparrow \infty$. This will yield a sequence $\{(V^j, W^j)\}_{j=1}^{\infty}$ such that $W^j \uparrow W^*$, a solution to \eqref{eq:SGL0_prob}. This algorithm is based on the quadratic penalty method \cite{nocedal2006numerical} and the penalty decomposition method \cite{lu2013sparse}. The algorithm is summarized in Algorithm \ref{alg:SGL0}.

\begin{algorithm}
\SetAlgoLined
    Initialize $V^1$ and $W^1$ with random entries; learning rate $\gamma$; regularization parameters $\lambda$ and $\beta$; and multiplier $\sigma > 1$. \\ 
Set $j \coloneqq 1$.\\
 \While{stopping criterion for outer loop not satisfied}{
Set $k \coloneqq 1$.\\
Set $W^{j,1} = W^j$ and $V^{j,1} = V^j$.\\
 \While{stopping criterion for inner loop not satisfied}{
 Update $W^{j,k+1}$ by Eq. \eqref{eq:gradient_descent}. \\
 Update $V^{j,k+1}$ by Eq. \eqref{eq:hard_threshold}.\\
 $k \coloneqq k+1$}
 Set $W^{j+1} = W^{j,k}$ and $V^{j+1} = V^{j,k}$.\\
 Set $\beta \coloneqq \sigma \beta$.\\
 Set $j \coloneqq j+1$. 
 }
 \caption{Algorithm for Sparse Group L$_0$asso Regularization}
 \label{alg:SGL0}
\end{algorithm}
\subsection{Convergence Analysis}
To establish convergence for the proposed algorithm, we show that the accumulation point of the sequence generated by \eqref{eq:W_update}-\eqref{eq:V_update} is a block-coordinate minimizer, and  an accumulation point generated by Algorithm \ref{alg:SGL0} is a sparse feasible solution to \eqref{eq:SGL0_prob}. Unfortunately, this feasible solution may not be a local minimizer of \eqref{eq:SGL0_prob} because the loss function $\mathcal{L}(\cdot, \cdot)$ is  nonconvex. However, it was shown in \cite{dinh2018convergence} that a similar algorithm to \eqref{alg:SGL0}, but only for $\ell_0$ minimization, generates an approximate global solution with high probbility for a one-layer CNN with ReLu activation function.
\begin{mythm} \label{lemma:AM_convergence}
Let $\{(V^k, W^k)\}_{k=1}^{\infty}$ be a sequence generated by the alternating minimization algorithm \eqref{eq:W_update}-\eqref{eq:V_update}. If $(V^*, W^*)$ is an accumulation point of $\{( V^k, W^k)\}_{k=1}^{\infty}$, then $(V^*, W^*)$ is a block coordinate minimizer of \eqref{eq:relaxed_min_prob}. that is
\begin{align*}
    V^* &\in \argmin_{V^l} F_{\beta}(V, W^*) \\
    W_l^* &\in \argmin_{W_l} F_{\beta}(V^*, W_{<l}^*, W_l, W_{>l}^*) \text{ for } l=1, \ldots, L
\end{align*}
\end{mythm}
\begin{proof}
By \eqref{eq:W_update}- \eqref{eq:V_update}, we have
\begin{align}
    F_{\beta}(V^k, W^+_{\leq l}, W^k_{> l}) &\leq F_{\beta} (V^k, W^+_{< l}, W_l, W^k_{> l}) \label{eq:W_ineq} \\
    F_{\beta}(V^+, W^+) &\leq F_{\beta} (V, W^+) \label{eq:V_ineq}
\end{align}
for all $W_l$, $l=1, \ldots, L$, and $V$, so it follows after some computation that
\begin{align} \label{eq:ineq_result}F_{\beta}(V^+, W^+) \leq F_{\beta}(V^k, W^k) 
\end{align}
for each $k \in \mathbb{N}$. Hence, $\{F_{\beta}(V^k, W^k)\}_{k=1}^{\infty}$ is nonincreasing.  Since $F_{\beta}(V^k, W^k) \geq 0$ for all $k \in \mathbb{N}$, its limit $\lim_{k \rightarrow \infty}F_{\beta}(V^k, W^k)$ exists. From \eqref{eq:W_ineq}, we have 
\begin{align*}
    F_{\beta}(V^k, W_{\leq l}^+, W_{> l}^k) \leq F_{\beta}(V^k, W_{<l}^+, W_{\geq l}^k)
\end{align*}
for each $l$. Because $F_{\beta}$ is continuous with respect to $W_l$, applying the limit gives us
\begin{align} \label{eq:limit_eq}
  \lim_{k \rightarrow \infty} F_{\beta}(V^k, W_{\leq l}^+, W_{> l}^k)
= \lim_{k \rightarrow \infty} F_{\beta}(V^k, W_{<l}^+, W_{\geq l}^k). 
\end{align}
Since $(V^*, W^*)$ is an accumulation point of $\{( V^k, W^k)\}_{k=1}^{\infty}$, there exists a subsequence $K$ such that $\lim_{k \in K \rightarrow \infty} (V^k, W^k) = (V^*,W^*)$. If $\lim_{k \in K \rightarrow \infty} V^k = V^*$, there exists $k' \in K$ such that $k \geq k'$ implies $\|V_l^k\|_0 \geq \|V_l^*\|_0$ for each $l = 1, \ldots, L$. As a result, we obtain
\begin{align*}
    F_{\beta}(V, W^{k+1}) &\geq F_{\beta}(V^{k+1}, W^{k+1}) \\
    &\geq \tilde{\mathcal{L}}(W^{k+1})\\ &+   \sum_{l =1}^L \Big[\lambda\left( \mathcal{R}_{GL}(W_l^{k+1})+\|V_l^*\|_0 \right)\\ &+ \frac{\beta}{2} \|V_l^{k+1}-W_l^{k+1}\|_2^2 \Big]
\end{align*}
for $k \geq k'$ from \eqref{eq:V_ineq}. Using continuity, except for the $\ell_0$ term, and letting $k\in K \rightarrow \infty$, we obtain
\begin{align}
F_{\beta}(V, W^*)\geq F_{\beta}(V^*, W^*).
\end{align} For notational convenience, let
\begin{align}\label{eq:V_minimizer}
    \tilde{\mathcal{R}}_{\lambda, \beta}(V,W) \coloneqq \lambda \mathcal{R}_{GL}(W) + \frac{\beta}{2} \|W-V\|_2^2.
\end{align}From \eqref{eq:W_ineq}, we have
\begin{align}
\label{eq:ineq1}
    &\tilde{\mathcal{L}}(W_{<l}^+, W_l, W_{>l}^k) + \sum_{j < l} \tilde{\mathcal{R}}_{\lambda, \beta}(W_j^{+}, V_j^k) + \tilde{\mathcal{R}}_{\lambda, \beta}(W_l, V_l^k) \\ \nonumber
    &+ \sum_{j > l} \tilde{\mathcal{R}}_{\lambda, \beta}(W_j^{k}, V_j^k) \\ \nonumber
    &= F_{\beta}(V, W_{< l}^+, W_l, W_{> l}^k) - \lambda \sum_{l=1}^L \|V_l^k\|_0 \\ \nonumber
    &\geq F_{\beta}(V, W_{\leq l}^+, W_{> l}^k) - \lambda \sum_{l=1}^L \|V_l^k\|_0 \\ \nonumber
    &=\tilde{\mathcal{L}}(W_{\leq l}^+, W_{>l}^k) + \sum_{j \leq l} \tilde{\mathcal{R}}_{\lambda, \beta}(W_j^{+}, V_j^k)\\ \nonumber &+ \sum_{j > l} \tilde{\mathcal{R}}_{\lambda, \beta}(W_j^{k}, V_j^k) 
\end{align}
for all $k \in K$. Because $\lim_{k \in K \rightarrow \infty} V^k$ exists, the sequence $\{V^k\}_{k \in K}$ is bounded, which implies that $\{\|V^k\|_0\}_{k \in K}$ is bounded as well. Hence, there exists a further subsequence $\overline{K} \subset K$ such that $\lim_{k \in \overline{K} \rightarrow \infty} \|V^k\|_0$ exists. As a result, For each $l = 1, \ldots, L$, we have that $\lim_{k \in \overline{K}} \|V_l^k\|_0$ exists. So, we obtain
\begingroup
\allowdisplaybreaks
\begin{align} \label{eq:eq1}
    &\lim_{k \in \overline{K} \rightarrow \infty} \tilde{\mathcal{L}}(W_{\leq l}^+, W_{>l}^k) + \sum_{j \leq l} \tilde{\mathcal{R}}_{\lambda, \beta}(W_j^{+}, V_j^k)\\ \nonumber &+ \sum_{j > l} \tilde{\mathcal{R}}_{\lambda, \beta}(W_j^{k}, V_j^k) \\ \nonumber &=\lim_{k \in  \overline{K} \rightarrow \infty} F_{\beta}(V, W_{\leq l}^+, W_{> l}^k) - \lambda \sum_{l=1}^L \|V_l^k\|_0\\ \nonumber
    &=\lim_{k \in  \overline{K} \rightarrow \infty} F_{\beta}(V, W_{\leq l}^+, W_{> l}^k) - \lim_{k \in  \overline{K} \rightarrow \infty} \lambda \sum_{l=1}^L \|V_l^k\|_0\\ \nonumber
    &=\lim_{k \in  \overline{K} \rightarrow \infty} F_{\beta}(V^k, W_{<l}^+, W_{\geq l}^k)-\lim_{k \in  \overline{K} \rightarrow \infty} \lambda \sum_{l=1}^L \|V_l^k\|_0 \\  \nonumber
    &= \ldots \\ \nonumber
    &= \lim_{k \in  \overline{K} \rightarrow \infty} F_{\beta}(V^k, W^k)-\lim_{k \in  \overline{K} \rightarrow \infty}\lambda \sum_{l=1}^L \|V_l^k\|_0 \\ \nonumber
    &= \lim_{k \in \overline{K} \rightarrow \infty} F_{\beta}(V^k, W^k) - \lambda \sum_{l=1}^L \|V_l^k\|_0 \\ \nonumber
    &=\lim_{k \in \overline{K} \rightarrow \infty} \tilde{\mathcal{L}}(W^k) + \sum_{l=1}^L \tilde{\mathcal{R}}_{\lambda, \beta}(W_l^{k}, V_l^k) \\ \nonumber
    &=\tilde{\mathcal{L}}(W^*) + \sum_{l=1}^L \tilde{\mathcal{R}}_{\lambda, \beta}(W_l^*, V_l^*)
\end{align}
after applying \eqref{eq:limit_eq}. Taking the limit over the subsequence $ \overline{K}$ in \eqref{eq:ineq1} and applying \eqref{eq:eq1}, we obtain 
\begin{align}
&\tilde{\mathcal{L}}(W_{<l}^*, W_l, W_{>l}^*) + \sum_{j \neq l} \tilde{\mathcal{R}}_{\lambda, \beta}(W_j^{*}, V_j^*)\\ &+ \tilde{\mathcal{R}}_{\lambda, \beta}(W_l, V_l^*) \geq \nonumber \tilde{\mathcal{L}}(W^*) + \sum_{l=1}^L \tilde{\mathcal{R}}_{\lambda, \beta}(W_l^*, V_l^*)
\end{align}
\endgroup
Adding $\sum_{l=1}^L \|V_l^*\|_0$ on both sides yields 
\begin{align}\label{eq:W_minimizer}
F_{\beta}(V^*, W^*_{<l}, W_l, W^*>l) \geq F_{\beta}(V^*, W^*).
\end{align} 
By \eqref{eq:V_minimizer} and \eqref{eq:W_minimizer}, $(V^*,W^*)$ is a block coordinate minimizer. 
\end{proof}
\begin{mythm}
Let $\{(V^k, W^k, \beta_k)\}_{k=1}^{\infty}$ be a sequence generated by Algorithm \ref{alg:SGL0}. Suppose that $\{F_{\beta_k}(V^k, W^k)\}_{k=1}^{\infty}$ is uniformly bounded.  If $(V^*, W^*)$ is an accumulation point of $\{V^k, W^k)\}_{k=1}^{\infty}$, then $V^* = W^*$  and $W^*$ is a feasible solution to \eqref{eq:SGL0_prob}.
\end{mythm}
\begin{proof}
Because $(V^*, W^*)$ is an accumulation point, there exists a subsequence $K$ such that $\lim_{k \in K \rightarrow \infty} (V^k, W^k) =(V^*, W^*)$. If $\{F_{\beta_k}(V^k, W^k)\}_{k=1}^{\infty}$ is uniformly bounded, there exists $M$ such that $F_{\beta_k}(V^k, W^k) \leq M$ for all $k \in \mathbb{N}$. After some algebraic manipulation, we should obtain
\begin{align}
\sum_{l = 1}^L \|V_l^k - W_l^k\|_2^2 \leq \frac{2}{\beta_k} M,
\end{align}
where $M$ is some positive constant equals to the total number of weight parameters in $W$. Taking the limit over $k \in K $, we have
\begin{align*}
    \sum_{l = 1}^L \|V_l^* - W_l^*\|_2^2 = 0,
\end{align*}
which follows that $V^* = W^*$. As a result, $W^*$ is a feasible solution to \eqref{eq:SGL0_prob}.
\end{proof}
\section{Experiments}
We compare the proposed sparse group l$_0$asso regularization against four other methods as baselines: group lasso, sparse group lasso, combined group and exclusive sparsity (CGES) proposed in \cite{yoon2017combined}, and the group variant of $\ell_0$ regularization proposed in \cite{Louizos2017LearningSN}. For the group terms, the weights are grouped together based on the filters or output channels, which we will refer to as neurons. We apply these methods on the following image datasets: MNIST \cite{lecun1998gradient} using the LeNet-5-Caffe \cite{jia2014caffe} and CIFAR 10/100 \cite{krizhevsky2009learning} using wide residual networks \cite{zagoruyko2016wide}. Because the optimization algorithms do not drive most, if not all, the weights and neurons to zeroes, we have to set them to zeroes when their values are below a certain threshold. In our experiments, if the absolute weights are below $10^{-5}$, we set them to zeroes. Then, {\bf weight sparsity} is defined to be {\it the percentage of zero weights with respect to the total number of weights trained in the network}. If the normalized sum of the absolute values of the weights of the neuron is less than $10^{-5}$, then the weights of the neuron are set to zeroes. {\bf Neuron sparsity} is defined to be {\it the percentage of neurons whose weights are zeroes with respect to the total number of neurons in the network}. 
\subsection{MNIST Classification}
The MNIST dataset consists of 60k training images and 10k test images. It is trained on Lenet-5-Caffe, which has four layers with 1,370 total neurons and 431,080 total weight parameters. All layers of the network are applied with strictly the same type of regularization. No other regularization methods (e.g., dropout and  batch normalization) are used. The network is optimized using Adam \cite{kingma2014adam} with initial learning rate 0.001. For every 40 epochs, the learning rate decays by a factor of 0.1.  We set the regularization parameter $\lambda = 0.1/60000$. For sparse group l$_0$asso, we set $\beta = 2.5/60000$, and for every 40 epochs, it increases by a factor of 1.25. The network is trained for 200 epochs across 5 runs.  

\begin{table*}[h!]
    \centering
    \resizebox{\textwidth}{!}{
    \begin{tabular}{l|c|c|c}
         Method & Mean Weight Sparsity (\%) [Std (\%)] & Mean Neuron Sparsity (\%)[Std (\%)] & Test Error (\%) [Std (\%)]  \\
         \hline
         $\ell_0$ \cite{Louizos2017LearningSN} & 0.02 [$<0.01$] & 0 [0] & 0.69 [0.02]\\
         CGES & 94.12 [0.26] & 39.33 [1.61]  & 0.65 [0.04] \\
         group lasso & 88.38 [0.49]& 69.39 [0.64]  & 0.76 [0.02] \\
         sparse group lasso & 93.50 [0.13] & 73.52 [0.49] & 0.77 [0.03] \\
         sparse group $\ell_0$asso (proposed) & 89.27 [0.46] & 68.25 [0.49] & 0.67 [0.02]
    \end{tabular}}
    \caption{Comparison of the baseline methods and sparse group $l_0$asso regularization method on LeNet-5-Caffe trained on MNIST. Mean weight sparsity, mean neuron sparsity, and mean test error across 5 runs after 200 epochs are shown. $N= 60000$, the number of training points.}
    \label{tab:MNIST}
\end{table*}

Table \ref{tab:MNIST} reports the mean results for weight sparsity, neuron sparsity, and test error obtained at the end of the runs. The $\ell_0$ regularization method barely sparsifies the network. On the other hand, CGES obtains the lowest mean test error with the largest mean weight sparsity, but its mean neuron sparsity is not as high as group lasso, sparse group lasso, and sparse group l$_0$asso. The largest mean neuron sparsity is attained by sparse group lasso, but its corresponding test error is worse than the other methods. Sparse group l$_0$asso attains comparable mean weight and neuron sparsity as group lasso and sparse group lasso but with lower test error. Therefore, the proposed regularization is able to balance accuracy with both weight and neuron sparsity better than the baseline methods. 

\subsection{CIFAR Classification}
CIFAR 10/100 is a dataset that has 10/100 classes split into 50k training images and 10k test images. The dataset is trained on wide residual networks, specifically WRN-28-10. WRN-28-10 has approximately 36,500,000 weight parameters and 10,736 neurons. The network is optimized using stochastic gradient descent with initial learning rate 0.1. After every 60 epochs, learning rate decays by a factor of 0.2. Strictly the same type of regularization is applied to the weights of the hidden layer where dropout is utilized in the residual block. We vary the regularization parameter $\lambda = \alpha/50000$ by training the model on $\alpha \in \{0.01, 0.05, 0.1, 0.2, 0.5\}$. For sparse group l$_0$asso, we set $\beta = 25 \alpha/50000$ initially and it increases by a factor of 1.25 for every 20 epochs. The network is trained for 200 epochs across 5 runs. Note that we exclude $\ell_0$ regularization by Louizos \etal \cite{Louizos2017LearningSN} because we found the method to be unstable when $\alpha \geq 0.1$. The results are shown in Figures \ref{fig:cifar10} and \ref{fig:cifar100} for CIFAR 10 and CIFAR 100, respectively.

According to Figure \ref{fig:cifar10}, CGES outperforms the other methods when $\alpha = 0.01$ for both sparsity and test error. However, sparsity levels  stabilize after when $\alpha = 0.1$. Sparse group lasso attains the highest mean weight and neuron sparsity when $\alpha \geq 0.01$. Group lasso and sparse group l$_0$asso have comparable mean weight and neuron sparsity levels, but sparse group l$_0$asso outperforms the other methods in terms of test error when $\alpha \geq 0.05$. 

Figure \ref{fig:cifar100} shows that the results for CIFAR 100 are similar to the results for CIFAR 10. CGES has better weight sparsity when $\alpha \leq 0.1$, but it has the least neuron sparsity when $\alpha \geq 0.2$. Sparsity levels for CGES appear to stabilize after when $\alpha = 0.1$. Sparse group lasso attains the highest mean weight and neuron sparsity for $\alpha \geq 0.2$. The proposed method sparse group l$_0$asso has comparable weight and neuron sparsity as group lasso, but it has the lowest mean test error when $\alpha \in \{0.05, 0.1, 0.2\}$. 

Overall, the results demonstrate that sparse group l$_0$asso maintains superior test accuracy with similar sparsity levels as group lasso when trading accuracy for sparsity as $\alpha$ increases.
\begin{figure*}[h!!!]%
\centering
\begin{subfigure}{0.7\columnwidth}
\includegraphics[width=\columnwidth]{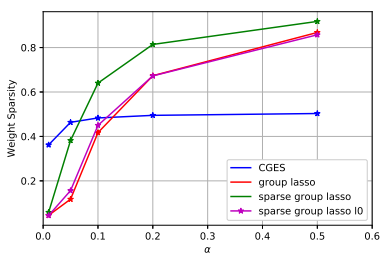}%
\caption{Mean weight sparsity}%
\label{subfig:weight_sparsity_c10}%
\end{subfigure}\hfill%
\begin{subfigure}{0.7\columnwidth}
\includegraphics[width=\columnwidth]{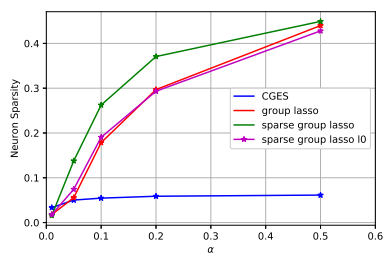}%
\caption{Mean neuron sparsity}%
\label{subfig:neuron_sparsity_c10}%
\end{subfigure}\hfill%
\begin{subfigure}{0.7\columnwidth}
\includegraphics[width=\columnwidth]{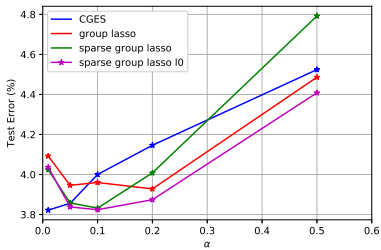}%
\caption{Mean test error}%
\label{subfig:test_err_c10}%
\end{subfigure}%
\caption{Mean results for CIFAR-10 on WRN-28-10 across 5 runs when varying the regularization parameter$\lambda = \alpha/50000$ when $\alpha \in \{0.01, 0.02, 0.1, 0.2, 0.5\}$.}
\label{fig:cifar10}
\end{figure*}
\begin{figure*}[t!!!]%
\centering
\begin{subfigure}{0.7\columnwidth}
\includegraphics[width=\columnwidth]{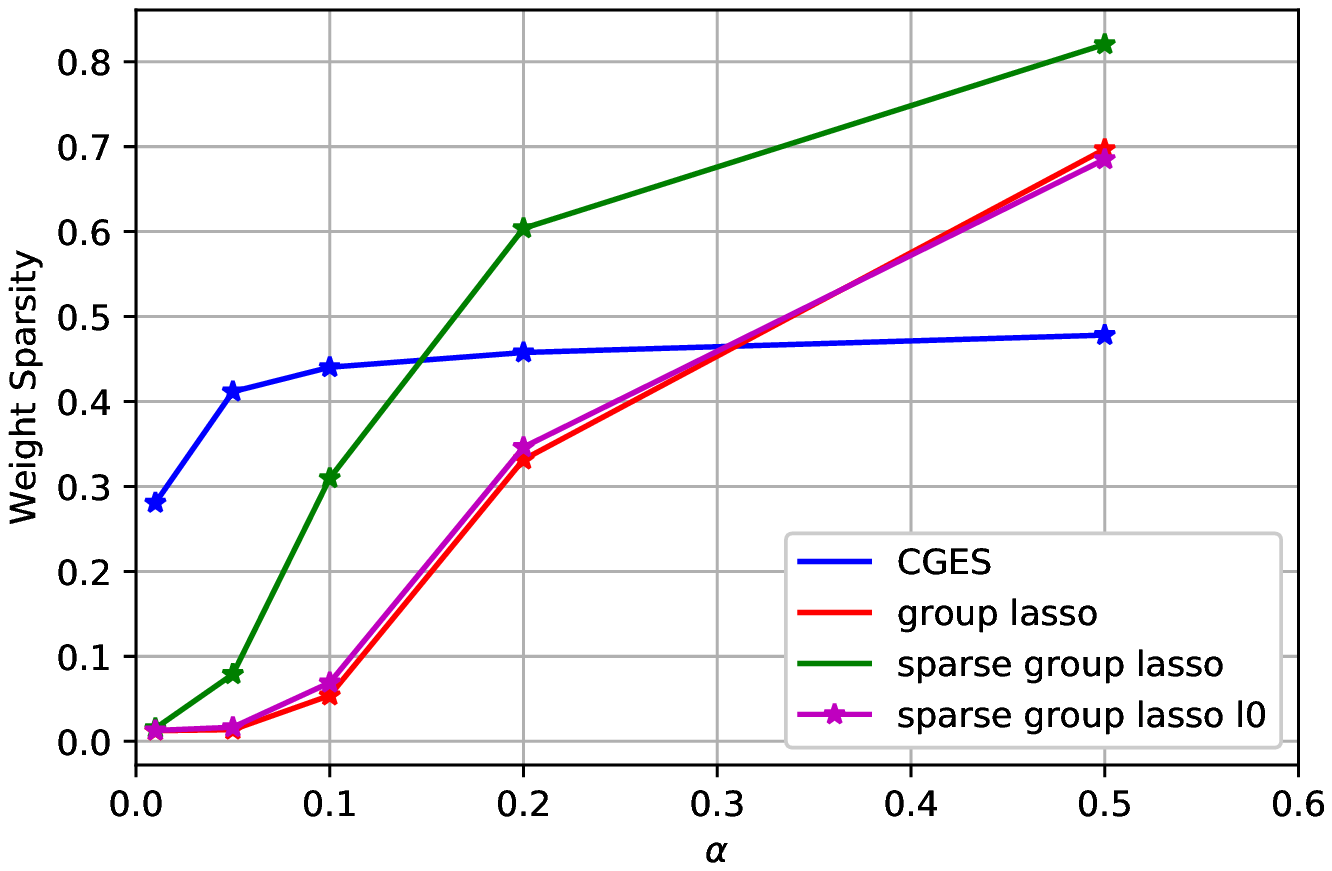}%
\caption{Mean weight sparsity}%
\label{subfig:weight_sparsity_c100}%
\end{subfigure}\hfill%
\begin{subfigure}{0.7\columnwidth}
\includegraphics[width=\columnwidth]{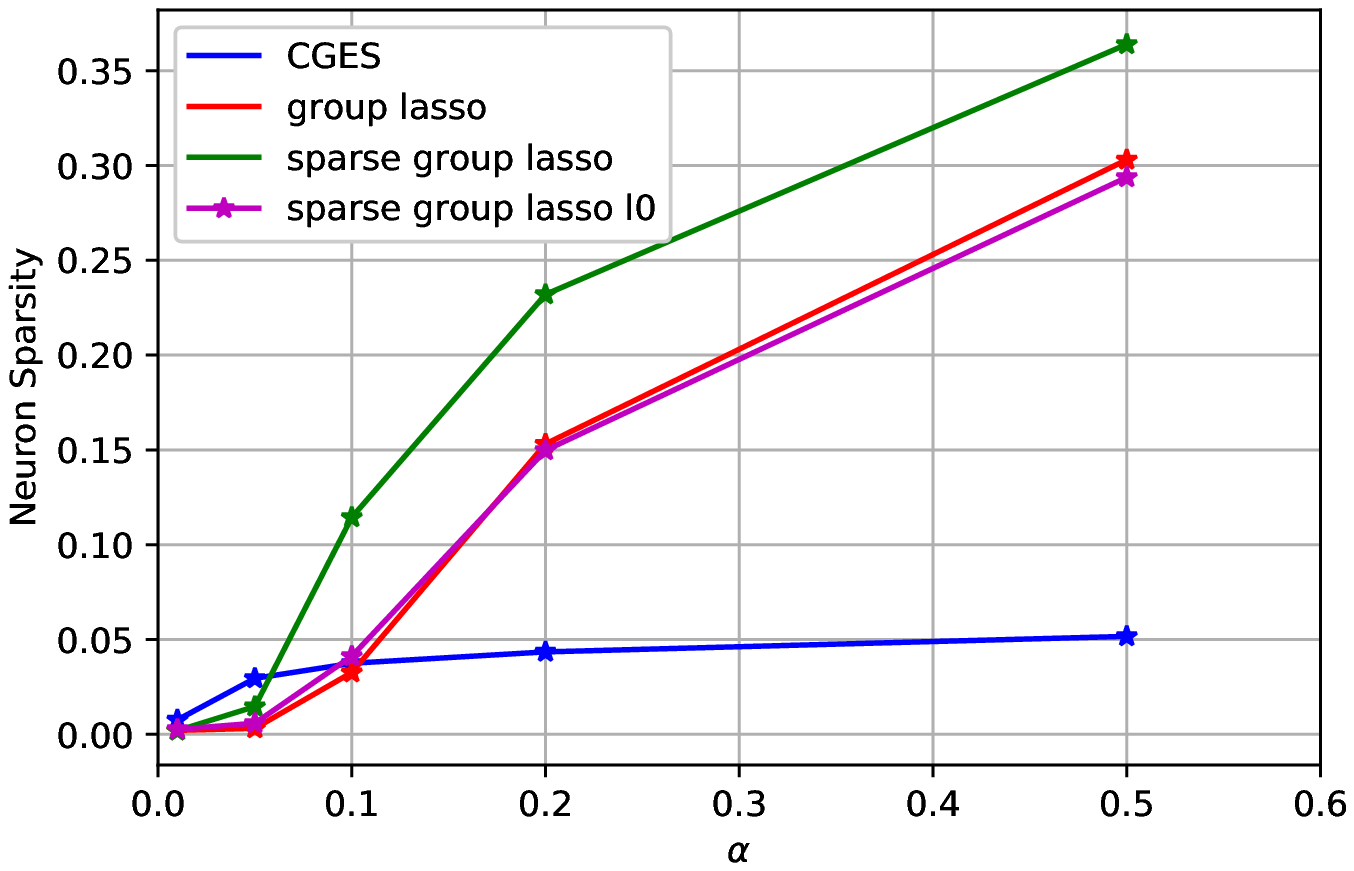}%
\caption{Mean neuron sparsity}%
\label{subfig:neuron_sparsity_c100}%
\end{subfigure}\hfill%
\begin{subfigure}{0.7\columnwidth}
\includegraphics[width=\columnwidth]{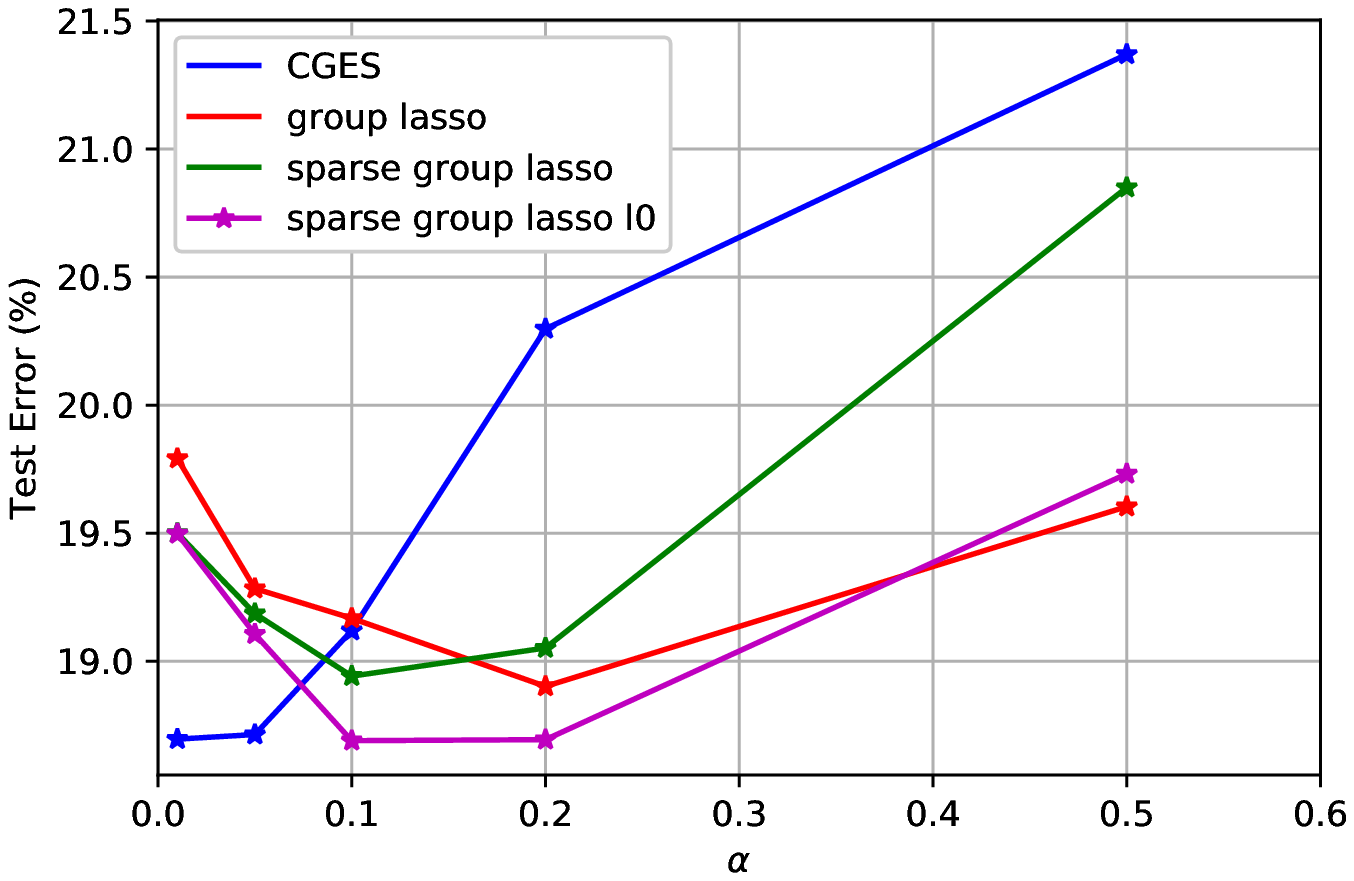}%
\caption{Mean test error)}%
\label{subfig:test_err_c100}%
\end{subfigure}%
\caption{Mean results for CIFAR-100 on WRN-28-10 across 5 runs when varying the regularization parameter $\lambda = \alpha/50000$ when $\alpha \in \{0.01, 0.02, 0.1, 0.2, 0.5\}$.}
\label{fig:cifar100}
\end{figure*}
\section{Conclusion and Future Work}
In this work, we propose sparse group l$_0$asso, a new variant of sparse group lasso where the $\ell_1$ norm on the weight parameters is replaced with the $\ell_0$ norm. We develop a new algorithm to optimize loss functions regularized with sparse group l$_0$asso for DNNs in order to attain a sparse network with competitive accuracy. We compare our method with various baseline methods on MNIST and CIFAR 10/100 on different CNNs. The experimental results demonstrate that in general, sparse group l$_0$asso attains similar weight and neuron sparsity as group lasso while maintaining competitive accuracy.  

For our future work, we plan to extend our proposed variant to other nonconvex penalties, such as $\ell_1 - \ell_2$, transformed $\ell_1$, and $\ell_{1/2}$. We will examine these nonconvex sparse group lasso methods on various experiments, not only on MNIST and CIFAR 10/100 but also on Tiny Imagenet and Street View House Number trained on different networks such as MobileNetv2 \cite{sandler2018mobilenetv2}. In addition, we might investigate in developing an alternating direction method of multipliers algorithm \cite{boyd2011distributed} as an alternative to the algorithm developed in this paper.
\section{Acknowledgement}
The work was partially supported by NSF grant IIS-1632935, DMS- 1854434, a Qualcomm Faculty Award, and Qualcomm AI Research.
\newpage
{\small
\bibliographystyle{ieee}
\bibliography{egbib}
}

\end{document}